\documentclass[10pt, a4paper]{article}
\usepackage[margin=1.1in]{geometry} 

\usepackage{natbib}
\usepackage{amssymb}
\usepackage{amsmath}
\usepackage{amsthm}
\usepackage{bbm}
\usepackage{times}
\usepackage[normalem]{ulem}

\usepackage{hyperref}

\usepackage[colorinlistoftodos, textwidth=4cm, shadow]{todonotes}

\usepackage{array}
\usepackage{multirow}

\newcommand{\ST}[1]{{\color{black}#1}}

\renewcommand{\S}{\mathcal{S}}
\newcommand{\A}{\mathcal{A}}

\newcommand{\R}{\mathbb{R}}
\renewcommand{\P}{\mathbb{P}}
\newcommand{\E}{\mathbb{E}}
\newcommand{\NN}{\mathbb{N}}

\newcommand{\RSQVI}{RS-QVI}

\newcommand{\argmax}{\arg\!\max}

\newcommand{\cA}{\mathcal{A}}

\newcommand{\cO}{\mathcal{O}}

\newcommand{\cS}{\mathcal{S}}
\newcommand{\cT}{\mathcal{T}}

\newcommand{\cX}{\mathcal{X}}

\newtheorem{theorem}{Theorem}

\newtheorem{definition}{Definition}
\newtheorem{lemma}{Lemma}

\newtheorem{remark}{Remark}

\usepackage[linesnumbered,ruled,vlined]{algorithm2e}
\usepackage{mathtools}

\usepackage[toc,page,header]{appendix}
\usepackage{minitoc}

\title{Near-Optimal Sample Complexity for Recursive Entropic Risk Reinforcement Learning with a Generative Model}

\author{Amirparsa Bahrami\footnote{School of Electrical Engineering, Sharif University of Technology. Email: {\tt\small amirparsa2828@gmail.com}.}
\and Oliver Mortensen\footnote{Department of Computer Science, University of Copenhagen. Email: {\tt\small olmo@di.ku.dk}.}
        \and Mohammad Sadegh Talebi\footnote{Department of Computer Science, University of Copenhagen. Email: {\tt\small sadegh.talebi@di.ku.dk}.}}

\date{\today}

\begin{document}
\maketitle

\begin{abstract}
In this paper, we study the sample complexities of value and policy learning in finite discounted Markov decision processes (MDPs) under recursive entropic risk preferences with risk parameter \(\beta\neq 0\), assuming access to a generative model of the MDP. We provide a refined analysis of model-based risk-sensitive Q-value iteration (MB-RS-QVI), a plug-in model-based method introduced in prior work, and derive \((\varepsilon,\delta)\)-PAC guarantees for both learning the optimal \(Q\)-value function and an \(\varepsilon\)-optimal policy. Our bounds improve the exponential dependence on the effective horizon \(1/(1-\gamma)\) compared with the best existing guarantees for this setting. In particular, they match the existing lower bounds in their exponential dependence on \(|\beta|/(1-\gamma)\), as well as in \(S\), \(A\), \(\varepsilon\), and \(|\beta|\), up to logarithmic factors. Consequently, our analysis removes the exponential gap between the previously known upper and lower bounds, leaving only a polynomial gap in the effective horizon.
\end{abstract}

\section{Introduction}
\label{Section:Introduction}

Reinforcement learning (RL) traditionally deals with \emph{risk-neutral} decision making under uncertainty, where the underlying objective is the expected return, namely, the (possibly discounted) sum of rewards \citep{sutton1998reinforcement}. 
However, this is arguably inadequate for many high-stakes application domains, such as operations research \citep{delage2010percentile}, healthcare \citep{ernst2006clinical}, and transportation \citep{kamran2020risk}. Such applications must account for the variability of returns and the associated risks.  One principled approach to addressing this limitation is to optimize a risk measure of the return distribution. 
Notable risk measures considered in the literature include mean-variance \citep{li2000optimal}, value-at-risk (VaR) \citep{dempster2002risk}, Conditional VaR (CVaR) \citep{shapiro2021LNs}, entropic risk measure (ERM)  \citep{howard1972risk}, and entropic VaR (EVaR) \citep{ahmadi2012entropic}, all of which have been applied in numerous application domains. Among these, CVaR and ERM have received great attention from the community for modeling risk-sensitivity in MDPs \citep{chow2014algorithms,bisi2022risk,bauerle_2011_markov,howard1972risk,borkar2002risk,hau2023entropic,fei2020risk}. However, much of the existing literature studies undiscounted problems, despite the relevance of discounted MDPs; see, e.g., \cite{bauerle_more_2014,hau2023entropic,mortensen2026recursive} for notable exceptions.

In this paper, we study risk-sensitive RL in discounted MDPs, assuming access to a generative model of the MDP, which is a simulator that generates samples from the true MDP for arbitrary state-action pairs. We consider a decision-making scenario where the agent's objective is defined using the entropic risk measure (ERM), \emph{recursively} applied to the reward-to-go at every step $t$ to the decision making. This is also called iterated, nested, or dynamic formulation in the various communities \citep{asienkiewicz2017note,bauerle_markov_2022,deng2025near,mortensen2026recursive}. This formulation stands in contrast to the \emph{non-recursive} (also called non-iterated or static) formulation, where the risk measure is directly applied to the total return (e.g., $\sum_{t=0}^\infty \gamma^t r_t$ in the discounted case) 
\citep{borkar2002q,borkar2002risk,hau2023dynamic}. The two formulations are not directly comparable and each may be preferred depending on the underlying application; we refer to, e.g., \citep{deng2025near,mortensen2026recursive}, for further discussion. 

Specifically, we consider risk-sensitive RL in finite tabular discounted MDPs under recursive ERM, assuming access to a generative model of the MDP --i.e., a simulator that generates samples from the true MDP for arbitrary state-action pairs. 
Learning performance is assessed in terms of sample complexity, defined as the total number $T$ of samples required, for given $(\varepsilon,\delta)$, to obtain either an $\varepsilon$-optimal policy (\emph{policy learning}) or an $\varepsilon$-close approximation of the optimal Q-value in the max-norm (\emph{value learning}), with probability exceeding $1-\delta$.

Our main contribution is to provide a refined sample complexity bound for Model-Based ERM Q-Value Iteration (MB-\RSQVI) \citep{mortensen2026recursive}, which is a model-based approach relying on a plug-in estimator. 
We show that the sample complexity of MB-\RSQVI\ for value learning (Theorem \ref{thm:MBRSQVI_UB_Q}) and policy learning (Theorem \ref{thm:MBRSQVI_UB_policy}) scales as\footnote{
Throughout the paper, $\widetilde{\mathcal O}$ and $\widetilde{\Omega}$ suppress logarithmic factors in the relevant problem parameters.} 
\begin{align*}
    \widetilde{O}\bigg(\frac{SA}{\varepsilon^2(1-\gamma)^2|\beta|^2}e^{|\beta|/(1-\gamma)}\bigg) \qquad\hbox{and}\qquad \widetilde{O}\bigg(\frac{SA}{\varepsilon^2(1-\gamma)^2|\beta|^2}\min\Big\{S,\frac{1}{(1-\gamma)^2}\Big\}e^{|\beta|/(1-\gamma)}\bigg)\,,
\end{align*}
respectively. 
%
These bounds hold for any discounted MDP with $S$ states, $A$ actions, discount factor $\gamma$, and risk parameter $\beta$, with $\beta>0$ (respectively, $\beta<0$) corresponding to a risk-averse (respectively, risk-seeking) agent; see Section \ref{Section:Background} for a precise definition. 
Moreover, the bounds are valid over the entire $\varepsilon$-range, namely $\varepsilon\in(0,\frac{1}{1-\gamma}]$. These bounds has an exponentially improved dependency on the effective horizon $1/(1-\gamma)$ when compared to the best existing bounds for value and policy learning under recursive ERM in discounted MDPs presented by \cite{mortensen2026recursive}. Our PAC bounds exhibit the same exponential dependence on the effective horizon, as predicted by lower bounds of \citep{mortensen2026recursive}
for value and policy learning, which scale as 
\begin{align*}
    \widetilde{\Omega}\bigg(\frac{SA}{\varepsilon^2|\beta|^2}e^{|\beta|/(1-\gamma)}\bigg). 
\end{align*}


We refer to  Table \ref{tab:bounds_summary} for a summary of our results and comparison to existing bounds.

\begin{table*}[!t]
\centering

\newcommand{\rowstrut}{\rule[-2ex]{0pt}{5.2ex}}

\resizebox{\textwidth}{!}{%
\begin{tabular}{|c|c|c|}
\hline
Problem & Upper Bound & Lower Bound \\ \hline

\multirow{2}{*}{Value Learning}
&
\rowstrut
$\widetilde{\mathcal{O}}\Big(
\frac{SA}{\varepsilon^2(1-\gamma)^2|\beta|^2}
e^{|\beta|/(1-\gamma)}
\Big)$
{[Thm.~\ref{thm:MBRSQVI_UB_Q}]}
&
\multirow{2}{*}{
\begin{tabular}[c]{@{}c@{}}
$\widetilde\Omega\Big(
\frac{SA}{\varepsilon^2|\beta|^2}
e^{|\beta|/(1-\gamma)}
\Big)$\\
{[\citealp{mortensen2026recursive}, Thm.~3]}
\end{tabular}
}
\\ \cline{2-2}

&
\rowstrut
$\widetilde{\mathcal{O}}\Big(
\frac{SA}{\varepsilon^2(1-\gamma)^2|\beta|^2}
e^{2|\beta|/(1-\gamma)}
\Big)$
{[\citealp{mortensen2026recursive}, Thm.~1]}
&
\\ \hline

\multirow{2}{*}{Policy Learning}
&
\rowstrut
$\widetilde{\mathcal{O}}\Big(
\frac{SA}{\varepsilon^2(1-\gamma)^2|\beta|^2}
\min\big\{S,\frac{1}{(1-\gamma)^2}\big\}
e^{|\beta|/(1-\gamma)}
\Big)$
{[Thm.~\ref{thm:MBRSQVI_UB_policy}]}
&
\multirow{2}{*}{
\begin{tabular}[c]{@{}c@{}}
$\widetilde\Omega\Big(
\frac{SA}{\varepsilon^2|\beta|^2}
e^{|\beta|/(1-\gamma)}
\Big)$\\
{[\citealp{mortensen2026recursive}, Thm.~4]}
\end{tabular}
}
\\ \cline{2-2}

&
\rowstrut
$\widetilde{\mathcal{O}}\Big(
\frac{SA}{\varepsilon^2(1-\gamma)^2|\beta|^2}
\min\big\{S,\frac{1}{(1-\gamma)^2}\big\}
e^{2|\beta|/(1-\gamma)}
\Big)$
{[\citealp{mortensen2026recursive}, Thm.~2]}
&
\\ \hline

\end{tabular}%
}

\caption{Summary and comparison of sample complexity bounds for recursive ERM in the small-$\varepsilon$ regime, where the quadratic terms in $1/\varepsilon$ dominate. $\beta$ denotes the risk parameter, where $\beta>0$ (respectively, $\beta<0$) corresponds to a risk-averse (respectively, risk-seeking) agent.}

\label{tab:bounds_summary}

\end{table*}

\section{Related Work}
\label{Section:RelatedWork}

\paragraph{Risk-neutral discounted RL.} There exists a rich literature on provably efficient RL in tabular discounted MDPs, covering a variety of settings such as  the generative setting \citep{kakade2003sample,gheshlaghi2013minimax,agarwal_model-based_2020,sidford2018variance,li2020breaking,jin2024truncated}, the offline (or batch) setting \citep{rashidinejad2021bridging,li2024settling}, and the online setting \citep{strehl_analysis_2008,lattimore_near-optimal_2014}. 
In the generative setting, notable work includes \citep{kakade2003sample,kearns1998phased_Q_learning,gheshlaghi2013minimax,sidford2018variance,wang2020randomized,li2024breaking,jin2024truncated}. 
\citet{gheshlaghi2013minimax} provide the first minimax-optimal sample complexity bounds of $\widetilde{\cal O}\big(\frac{SA}{\varepsilon^2(1-\gamma)^3}\big)$ for both value and policy learning, accompanied by a lower bound of $\widetilde{\Omega}\big(\frac{SA}{\varepsilon^2(1-\gamma)^3}\big)$ for value learning. The lower bound was argued to be valid for the policy learning setting as well (up to logarithmic factors). Among the cited papers above, \citep{sidford2018variance,wang2020randomized,jin2024truncated} study model-free methods. 
\citet{li2024breaking} derive an optimal bound valid for the entire  $\varepsilon$-range, using model-based methods constructed from the empirical MDP but with reward perturbations or conservative planning. Existing optimal sample-complexity analyses rely on techniques that crucially exploit the additive structure of the return. As a result, these techniques generally do not extend to risk-sensitive objectives.

\paragraph{Risk-sensitive RL.} 
A substantial literature studies decision making under risk measures in bandit and RL settings. For RL,  representative examples include work on CVaR \citep{deng2025near,du2023provably,chen2024provably,lam2022risk},  entropic risk \citep{borkar2002risk,moharrami2025policy,marthe2025efficient,hau2023entropic}, and  mean-variance risk \citep{huang2022achieving}. 
In particular, under recursive CVaR, \cite{deng2025near} analyze sample complexity in the generative setting and provide a lower bound; however, its analysis relies on structural properties specific to CVaR and does not extend to ERM.

ERM has also been widely studied across diverse settings \citep{borkar2002risk,sadana2024mitigating,fei2020risk,fei2021risk,fei2021exponential,marthe2025efficient,moharrami2025policy,hau2023entropic}. Most of these papers study the non-recursive setting. 
Under recursive ERM, most of the existing work studies online episodic RL in the regret setting (e.g., \citet{fei2020risk,fei2021exponential,fei2021risk,hu2023tighter,liang2024regret}). For the discounted setting, \cite{mortensen2026recursive} provide PAC upper and lower bounds in the generative setting, which is considered the closet work to the present paper. 
More recently, \citet{essakine26tight} study policy learning under the entropic risk measure in finite-horizon MDPs with a forward interaction model and derive tight sample-complexity guarantees for this setting. Their learning problem differs substantially from ours: it is finite-horizon and stage-dependent, and samples are collected along trajectories. In contrast, we consider stationary discounted infinite-horizon MDPs under a generative model.

Another line of work develops algorithms for entire families of risk measures. Two notable families studied in this context are coherent risk measures and optimized certainty equivalent (OCE) measures, both of which include important special cases such as mean–variance and CVaR. While ERM  is not coherent, it belongs to the OCE class for $\beta>0$
For coherent risk measures,  we refer to \citep{petrik2012approximate,tamar2015policy,lam2022risk,zhao2024ra}. For OCEs, 
notable existing results  include \citep{wang2025reductions,xu2023regret,rigter2023one,lee2025risk,mortensen2026OCE}. \cite{mortensen2026OCE} study learning under recursive OCEs in discounted MDPs and provide PAC upper and lower sample complexity bounds for both value and policy learning. However, their bounds do not improve those of ERM, when applied to this setting.  The other cited papers either do not provide provably sample-efficient learning guarantees or do not study discounted MDPs. Among the closest works,   \cite{rigter2023one} consider offline RL in discounted MDPs under recursive OCE but do not provide sample-complexity guarantees.

\section{Background}
\label{Section:Background}

\paragraph{Notations.} For $n \in \NN$, let $[n] := \{1, \ldots, n\}$. $\mathbbmss 1_A$ denotes the indicator function of an event $A$. Given a set $\cX$, $\Delta(\cX)$ denotes the probability simplex over $\cX$. We use the convention that $\|\cdot \| := \|\cdot \|_\infty$ and explicitly use the subscript $\|\cdot \|_p$ when using $p$-norms for  $1\leq p<\infty$. The notation $L^\infty(\Omega, \mathcal{F},\mathbb{P})$ denotes the space of essentially bounded random variables on the probability space $(\Omega, \mathcal{F},\mathbb{P})$.

\subsection{Entropic Risk Preferences}
The entropic risk measure (ERM) is rooted in expected utility theory \citep{mas1995microeconomic}. Consider for $\beta\neq0$ the class of utility functions $u(t) = \frac{1}{\beta}(1-e^{-\beta t})$ defined for $t\in \mathbb{R}$. The utility $u$ is supposed to describe the preferences of some economic agent in the form of how much utility $u(t)$ she derives from some monetary quantity $t\in \mathbb{R}$. For any bounded random variable $X\in L^\infty(\Omega, \mathcal{F},\mathbb{P})$, it is easy to verify that the associated \emph{certainty equivalent} to $u$ is $u^{-1}(\E[u(X)]) = \frac{-1}{\beta}\log(\E[e^{-\beta X}])$, which expresses the amount of money that would give the same utility as that of entering in the bet given by the random variable $X$. We thus define the functional $\rho : L^\infty(\Omega,\mathcal{F},\mathbb{P}) \rightarrow \mathbb{R}$ by
\begin{align}
\label{eq:ERM_functional}
    \rho(X) :=\rho(X; \beta):=-\frac{1}{\beta}\log\big(\E[e^{-\beta X}]\big)\,.
\end{align}
We note that there is a lack of consensus regarding the sign of $\beta$ in the definition of ERM. We follow this convention considering its widespread use in the actuarial literature \citep{asienkiewicz2017note}.

It is clear that we recover the risk-neutral case for $\beta \rightarrow 0$: $\lim_{\beta\to 0} \rho(X) = \E[X]$. We note that $\rho$ admits the following properties: 
\begin{align}
    \label{monotonicity_Phi}
    &\rho(X) \leq \rho(Y), \quad \text{for any }X\leq Y,
    \\
    \label{consistency_Phi}
    &\rho(c) = c, \quad \text{for any }c\in \mathbb{R},
    \\
    \label{jensen_beta>0}
    & \rho(X) \leq \E[X], \quad \text{for }\beta>0,
    \\
    \label{jensen_beta<0}
    & \rho(X) \geq \E[X], \quad \text{for }\beta<0,
\end{align}
where properties (\ref{jensen_beta>0})-(\ref{jensen_beta<0}) follow from Jensen's inequality. Using $\rho$ as a measure of the preference for different random variables, it follows directly from (\ref{consistency_Phi})-(\ref{jensen_beta<0}) that $\rho(X)\leq \rho(\E[X])$ for $\beta>0$ and that $\rho(X)\geq \rho(\E[X])$ for $\beta<0$. It further shows that $\beta>0$ is associated with \emph{risk-aversion}, while $\beta<0$ is associated with a \emph{risk-seeking} behavior. It is well-known that  ERM, unlike CVaR, does not belong to the nice class of coherent risk measures;  we refer the reader to Appendix A in \citep{mortensen2026recursive} for a primer on risk measures.  

\subsection{Discounted Markov Decision Processes with Entropic Risk}

A finite, discounted infinite-horizon Markov decision process (MDP) is a $5$-tuple $M=(\S,\A,P,R,\gamma)$, where $\S = \{1,2,\ldots,S\}$ is the finite state space of size $S:=|\S|$, $\A = \{1,2,\ldots,A\}$ is the finite action space of size $A:=|\A|$, $P:\S\times \A\rightarrow \Delta(\S)$ is the transition probability function, $R:\S\times \A\rightarrow [0,1]$ is the deterministic reward function, and $\gamma \in (0,1)$ is the discount factor.  We write $P(s'|s, a)$ to denote the transition probability to state $s'\in \S$ under pair $(s,a)$, and often use $P_{s,a}(s')$ as a corresponding shorthand notation. Finally, a stationary deterministic policy is a map $\pi:\S\rightarrow \A$. 

The agents interaction with the MDP $M$ is as follows. At initialization of the process, $M$ is in some initial state $s_0\in \S$. At each time $t\geq 0$, the agent is in state $s_t\in \S$ and decides on an action $a_t\in \A$. The MDP generates a reward $r_t:=R(s_t,a_t)$ and a next-state $s_{t+1}\sim P(\cdot|s_t,a_t)$. The MDP moves to $s_{t+1}$ when the next time slot begins, and this process continues ad infinitum. This process yields a growing sequence $(s_t,a_t,r_t)_{t\ge 0}$. 

Let $v\in \mathbb{R}^S$ and $\pi:\S\rightarrow \A$ be a stationary deterministic policy. Inspired by the definition of ERM functional in Eq.~(\ref{eq:ERM_functional}), We  define $\rho_{s,a}:\mathbb{R}^S\rightarrow \mathbb{R}$ as 
\begin{align}
    \rho_{s,a}(v) = -\frac{1}{\beta}\log\Big(\E_{s'\sim P_{s,a}}[e^{-\beta v(s')}] \Big),
\end{align}
and slightly abusing the notation, we write $\rho_{s,\pi}$ when $a=\pi(s)$, i.e., $\rho_{s,\pi}:=\rho_{s,\pi(s)}$.

The agent's goal is to maximize a risk-sensitive objective of the collected rewards $(r_t)_{t\ge 0}$, depending on both $\gamma$ and $\rho$. In the classical setting,  the value of a policy $\pi$ is defined as discounted sum of rewards collected under $\pi$. 
Under recursive ERM, the objective is defined through the iteration of $\rho$ \citep{bauerle_markov_2024}. The value function $V^\pi$ and the state-action value function (or Q-value) $Q^\pi$ of a policy $\pi$ are defined informally as 
\begin{align*}
  V^\pi(s_0) & =r_0+\gamma\rho_{s_0,\pi(s_0)}\Big(r_1+\gamma\rho_{s_1,\pi(s_1)}\big(r_2 + \gamma \rho_{s_2,\pi(s_2)}(\ldots)\big)\Big)\,,
    \\
  Q^\pi(s_0,a) & =R(s_0,a)+\gamma\rho_{s_0,a}\Big(r_1+\gamma\rho_{s_1,\pi(s_1)}\Big(r_2 +\gamma \rho_{s_2,\pi(s_2)}(\ldots)\Big)\Big)\,,\quad \forall a\in \A,
\end{align*}
with $r_t:=R(s_t,\pi(s_t))$. For all $(s,a)$, let $V^*(s):=\sup_{\pi}V^\pi(s)$ and $Q^*(s,a):=\sup_\pi Q^\pi(s,a)$ denote the optimal value (at state $s$) and Q-value (at state-action $(s,a)$), respectively, where $\sup$ is taken over the set of all possible policies (potentially randomized and history-dependent). Any policy $\pi^*$ satisfying $V^{\pi^*}(s)=V^*(s)$ for all $s\in \S$ is called optimal, and any policy $\pi$ obeying $V^\pi(s)\ge V^*(s) - \varepsilon$ for a given $\varepsilon>0$ is called $\varepsilon$-optimal. 
%
As established in \cite{bauerle_markov_2024}, under a stationary $\pi$, 
\begin{align*}
    V^\pi(s)& =R(s,\pi(s))+\gamma \rho_{s,\pi(s)}(V^\pi), \quad &&\forall s\in \S.
    \\
    Q^\pi(s,a)& =\!R(s,a)+\gamma\rho_{s,a}\big(Q^\pi(\cdot, \pi(\cdot))\big), \quad &&\forall (s,a)\in \cS\times \cA.
\end{align*}
Furthermore, there exists a stationary deterministic optimal policy $\pi^*:\cS\to\cA$ that achieves $V^*(s)$ for all states $s$ simultaneously, and the optimal value functions satisfy the Bellman optimality equations:   
\begin{align*}
    V^*(s)& =\max_{a\in \A}\big(R(s,a)+\gamma \rho_{s,a}(V^*)\big), \quad &&\forall s\in \S.
    \\
    Q^*(s,a)& =\!R(s,a)+\gamma\rho_{s,a}\big(\max_{a'\in \A}Q^*(\cdot,a')\big), \quad &&\forall (s,a)\in \cS\times \cA.
\end{align*}
Let us 
introduce the operators $\mathcal{T}^\pi,\mathcal{T}:\mathbb{R}^{S\!\times\!A}\rightarrow \mathbb{R}^{S\times A}$, which for $f:\cS\!\times\!\cA \rightarrow \mathbb{R}$ are defined as 
\begin{align*}
    (\mathcal{T}^{\pi}f)(s,a) &:=R(s,a)+\gamma\rho_{s,a}\big(f(\cdot,\pi(\cdot))\big), \\ 
    (\mathcal T f)(s,a) & :=R(s,a)+\gamma\rho_{s,a}\big(\max_{a'\in \A}f(\cdot,a')\big). 
\end{align*}
Then, the Bellman equations  for the $Q$-function can therefore be written as $Q^* = \cT Q^*$ and $Q^\pi = \cT^\pi Q^\pi$. The operators $\mathcal{T}$ and $\mathcal{T}^\pi$ are $\gamma$-contractions with respect to the max-norm (see, \citep[Lemma~1]{mortensen2026recursive}). Hence, $Q^*$ can be efficiently approximated arbitrarily well by a value-iteration-type algorithm.

\paragraph{Learning Performance.}
We consider RL algorithms that aim to find an $\varepsilon$-optimal policy or an $\varepsilon$-optimal value function for input $\varepsilon>0$  under ERM, while having access to a generative model (or simulator) of the MDP. Precisely speaking, the generative model can produce a sample $s'\sim P_{s,a}$ for any queried state-action $(s,a)$. We consider two types of such algorithms, which we generically denote by $\mathcal U$: 
The first type outputs a $Q$-value $Q_T^{\mathcal{U}}:\S\times \A \rightarrow \R$, whereas the second outputs a policy $\pi_T^\mathcal{U}:\cS\rightarrow \A$ using $T$  samples.  
We evaluate the quality of an algorithm that outputs a $Q$-value by $\|Q^*-Q_T^\mathcal{U}\|$. For an algorithm that instead outputs a policy, we evaluate it in terms of $\|V^*-V^{\pi_T^\mathcal{U}}\|$. Often, we will suppress $T$ from the notation. This leads to the notion of $(\varepsilon,\delta)$-correct value and policy for input parameters $(\varepsilon, \delta)$\ST{, similar to risk-neutral RL}: 

\begin{definition}[$(\varepsilon,\delta)$-correct value and policy]
    An algorithm $\mathcal{U}$ that outputs a $Q$-value $Q^\mathcal{U}$ is called \emph{$(\varepsilon,\delta)$-value-correct} on a set of MDPs $\mathbb{M}$ if $\P(\|Q^*-Q^{\mathcal{U}}\|\leq \varepsilon)\geq 1-\delta$ for all $M\in \mathbb{M}$. Similarly, an algorithm $\mathcal{U}$ that outputs a policy $\pi^\mathcal{U}$ is called \emph{$(\varepsilon,\delta)$-policy-correct} on a set of MDPs $\mathbb{M}$ if $\P(\|V^*-V^{\mathcal{\pi^\mathcal{U}}}\|\leq \varepsilon)\geq 1-\delta$ for all $M\in \mathbb{M}$.
\end{definition}

The notion of $(\varepsilon,\delta)$-value-correctness yields a sample complexity notion in the case of \emph{value learning}, while $(\varepsilon,\delta)$-policy-correctness serves a similar role for \emph{policy learning}.

\section{Model-Based ERM Q-Value Iteration}
\label{Section:Model-basedVI}
In this section, we recall the MB-\RSQVI\ algorithm \citep{mortensen2026recursive}, which is a model-based algorithm for value and policy learning settings under ERM in the generative setting. 

MB-\RSQVI\ obtains $N$ samples from each state-action pair in $Z=\cS\times \cA$, which is done by making $T=NSA$ calls to the generative model. Let $\widehat{P}$ denote the plug-in estimator built using the $T=NSA$ independent samples obtained from the generative model; that is, for $(s,a,s')\in Z\times \S$, $\widehat{P}(s'|s,a) = \frac{n(s,a,s')}{N}$, where $n(s,a,s')$ denotes the number of times $s'$ was observed under the queried pair $(s,a)\in Z$. The model-based approach we describe relies on the empirical MDP formulated using $\widehat P$, $\widehat M=(\S,\A,R,\widehat{P},\gamma,\beta)$, but it is otherwise general in the sense that it can use any oracle that outputs an $\varepsilon$-optimal policy for any $\varepsilon>0$. 
Such an oracle can be implemented using a Q-value iteration method akin to that of the classical risk-neutral setting; see \citep[Lemma 2]{mortensen2026recursive}, recalled below as Lemma 1.

The pseudocode is presented in Algorithm \ref{alg:MB-RS-QVI}.

\begin{algorithm}[H]
\label{alg:MB-RS-QVI}
\SetAlgoLined
\SetKwInput{KwInput}{Input}                
\SetKwInput{KwOutput}{Output}              
\DontPrintSemicolon

  \KwInput{Number $k$ of iterations}  
  Initialization: $\forall (s,a)$ set $Q(s,a) =0$ \;
  Collect $T$ samples for each $(s,a)$ from the generative model \;
  Build $\widehat P_{s,a}$
  \For{$j=0,1,\ldots,k-1$}{
  \For{all $(s,a)\in \cS\times \cA$}{
$Q_{j+1}(s,a) = R(s,a)-\frac{\gamma}{\beta}\log\big(\E_{s'\sim \widehat P_{s,a}}\big[e^{-\beta \max_{a'}Q_j(s',a')}\big]\big)$ \;\vspace{-4mm}
  }
  }
  $\forall s\in \cS:$ $\pi_k(s) = \argmax_{a\in\cA}Q_k(s,a)$ \;
  \KwRet $Q_k$ and $\pi_k$

 
  \caption{MB-\RSQVI\ \citep{mortensen2026recursive}}
\end{algorithm}

\section{Sample Complexity Analysis of MB-\RSQVI}
\label{sec:sample_complexity_analysis}

In this section, we present a sample complexity analysis of MB-\RSQVI\ under both policy learning and value learning. Theorem \ref{thm:MBRSQVI_UB_Q} states such a result 
for the case of value learning, while Theorem \ref{thm:MBRSQVI_UB_policy} offers a bound for policy learning. 

\begin{theorem}[{Sample complexity, value learning}]\label{thm:MBRSQVI_UB_Q}
For any $\varepsilon>0$, $\delta\in (0,1)$, and any MDP $M$ with $S$ states and $A$ actions, if the learner makes
\begin{align*}
T \ge 8SA\log\!\bigg(\frac{2SA}{\delta}\bigg)\left(e^{|\beta|/(1-\gamma)}-1\right)
\max\Bigg\{&
1,
\frac{4\gamma^2
\left(1- e^{-|\beta|/(1-\gamma)}\right)}
{\varepsilon^2(1-\gamma)^2\beta^2},
\frac{2\gamma}
{3\varepsilon(1-\gamma)|\beta|}
\Bigg\}
\end{align*}
model calls to the generative model, then $\P(\|Q^*-Q_k\|\leq \varepsilon)\geq 1-\delta$.
\end{theorem}

For sufficiently small $\varepsilon$, the terms of order $\varepsilon^{-2}$ dominate, and the bound simplifies to
\[
T = \widetilde O\!\left(
\frac{SA\gamma^2}{\varepsilon^2(1-\gamma)^2\beta^2}
\left(e^{|\beta|/(1-\gamma)}+e^{-|\beta|/(1-\gamma)}-2\right)
\right).
\]

\begin{theorem}[{Sample complexity, policy learning}]
\label{thm:MBRSQVI_UB_policy}
    For any $\varepsilon>0, \delta\in(0,1)$, and any MDP $M$ with $S$ states and $A$ actions, if the learner makes
    \begin{align*}
    T \ge SA\left(e^{|\beta|/(1-\gamma)}-1\right)
    \min\Bigg\{&
    \log\!\bigg(\frac{4SA|\Pi|}{\delta}\bigg)
    \max\Bigg\{
    8,
    \frac{72\gamma^2
    \left(1- e^{-|\beta|/(1-\gamma)}\right)}
    {\varepsilon^2(1-\gamma)^2\beta^2},\,
    \frac{8\gamma}
    {\varepsilon(1-\gamma)|\beta|}
    \Bigg\},\\
    &\log\!\bigg(\frac{4SA}{\delta}\bigg)
    \max\Bigg\{
    8,
    \frac{128\gamma^4
    \left(1-e^{-|\beta|/(1-\gamma)}\right)}
    {\varepsilon^2(1-\gamma)^4\beta^2},\,
    \frac{32\gamma^2}
    {3\varepsilon(1-\gamma)^2|\beta|}
    \Bigg\}
    \Bigg\}
    \end{align*}
    model calls, then $\P(\|V^*-V^{\pi_k}\|\leq \varepsilon)\geq 1-\delta$. Here, $\Pi$ denotes the set of stationary deterministic policies. 
\end{theorem}


For sufficiently small $\varepsilon$, the terms of order $\varepsilon^{-2}$ dominate, and the bound simplifies to
\[
T = \widetilde O\!\left(
\frac{SA\gamma^2}{\varepsilon^2(1-\gamma)^2\beta^2}
\left(e^{|\beta|/(1-\gamma)}+e^{-|\beta|/(1-\gamma)}-2\right)
\min\left\{\log |\Pi|,\,\frac{\gamma^2}{(1-\gamma)^2}\right\}
\right).
\]
Furthermore, using the worst-case bound $|\Pi|\leq A^S$, this can be written as
\[
T = \widetilde O\!\left(
\frac{SA\gamma^2}{\varepsilon^2(1-\gamma)^2\beta^2}
\left(e^{|\beta|/(1-\gamma)}+e^{-|\beta|/(1-\gamma)}-2\right)
\min\left\{S,\,\frac{\gamma^2}{(1-\gamma)^2}\right\}
\right).
\]

\subsection{Comparison with Existing Bounds}
\label{sec:comparison_mortensen_talebi}

We now compare Theorems \ref{thm:MBRSQVI_UB_Q}--\ref{thm:MBRSQVI_UB_policy}
with the existing bounds of \citet{mortensen2026recursive}. We focus on the small-$\varepsilon$ regime, in which the terms of order $\varepsilon^{-2}$ dominate the sample complexity bounds.

\paragraph{Improvement over existing upper bounds.}
For value learning, \citet{mortensen2026recursive} establish an upper bound that scales as
\begin{align*}
\widetilde{\cal O}\!\left(
\frac{SA\gamma^2}{\varepsilon^2(1-\gamma)^2\beta^2}
\big(e^{|\beta|/(1-\gamma)}-1\big)^2
\right),
\end{align*}
whereas Theorem \ref{thm:MBRSQVI_UB_Q} yields
\begin{align*}
\widetilde{\cal O}\!\left(
\frac{SA\gamma^2}{\varepsilon^2(1-\gamma)^2\beta^2}\big(e^{|\beta|/(1-\gamma)}+e^{-|\beta|/(1-\gamma)}-2\big)
\right) = 
\widetilde{\cal O}\!\left(
\frac{SA\gamma^2}{\varepsilon^2(1-\gamma)^2\beta^2}e^{-|\beta|/(1-\gamma)}\big(e^{|\beta|/(1-\gamma)}-1\big)^2
\right).
\end{align*}
Therefore, the leading term is improved by a factor of $\cO(e^{|\beta|/(1-\gamma)})$. 
We note however that the numerical constant in the leading term in our bound is larger than that in \cite{mortensen2026recursive}. 

The same improvement holds for policy learning. Using the worst-case bound $|\Pi|\le A^S$ and suppressing logarithmic factors, in the small-$\varepsilon$ regime, the policy-learning upper bound of \citet{mortensen2026recursive} scales as
\begin{align*}
\widetilde{\cal O}\!\left(
\frac{SA\gamma^2}{\varepsilon^2(1-\gamma)^2\beta^2}
\big(e^{|\beta|/(1-\gamma)}-1\big)^2
\min\!\left\{S,\frac{\gamma^2}{(1-\gamma)^2}\right\}
\right),
\end{align*}
whereas Theorem \ref{thm:MBRSQVI_UB_policy} gives
\begin{align*}
\widetilde{\cal O}\!\left(
\frac{SA\gamma^2}{\varepsilon^2(1-\gamma)^2\beta^2}
\left(e^{|\beta|/(1-\gamma)}+e^{-|\beta|/(1-\gamma)}-2\right)
\min\!\left\{S,\frac{\gamma^2}{(1-\gamma)^2}\right\}
\right).
\end{align*}
Hence, our analysis improves the exponential dependence by a factor of $\cO(e^{|\beta|/(1-\gamma)})$. 


\paragraph{Comparison with \citet{essakine26tight}.}
More recently, \citet[Theorem~4]{essakine26tight} establish a sample-complexity upper bound for policy learning under ERM in finite-horizon MDPs under the forward interaction model. Here, \(H\) denotes the finite horizon, i.e., the number of decision stages in each episode. In the small-\(\varepsilon\) regime, their bound can be written as
\[
    \widetilde{\mathcal O}\!\left(
    \frac{SAH}{\varepsilon^2\beta^2}
    \big(e^{|\beta|H}-1\big)
    \right).
\]
Thus, under the correspondence $H\asymp (1-\gamma)^{-1}$, their result exhibits the same order of exponential dependence on the horizon and the risk parameter as our bound.

This correspondence, however, is only useful for comparing the scaling of the two results and does not allow their finite-horizon analysis to be transferred directly to our discounted infinite-horizon setting. A key ingredient in their analysis is that, after exponentiating the entropic value function, the finite-horizon Bellman recursion takes the form
\[
    e^{\beta V_h(s)}
    =
    e^{\beta r_h(s,a)}
    \mathbb{E}\!\left[
        e^{\beta V_{h+1}(S')}
        \,\middle|\, s,a
    \right],
\]
for a fixed action \(a\). Hence, the unknown transition kernel enters through a linear conditional expectation, which permits direct variance-sensitive control of the corresponding transition-estimation error. Their analysis further relies on stage-indexed data and confidence estimates collected along finite-horizon trajectories, together with a backward stopping argument initialized at the terminal stage and propagated from \(H\) to \(1\).

In the discounted recursive ERM setting, by contrast, exponentiation yields a Bellman recursion of the form
\[
    e^{\beta V(s)}
    =
    e^{\beta r(s,a)}
    \left(
        \mathbb{E}\!\left[
            e^{\beta V(S')}
            \,\middle|\,s,a
        \right]
    \right)^{\gamma},
\]
so the transition expectation is composed with the nonlinear map \(x\mapsto x^\gamma\). Moreover, our setting is stationary and uses a generative model, where samples can be queried independently from arbitrary state-action pairs and all samples contribute to estimating the same transition kernel; there is also no finite terminal stage from which a backward stopping recursion can be initialized. Consequently, both the confidence-bound construction and the backward stopping argument of \citet{essakine26tight} do not apply directly to the discounted infinite-horizon setting. Replacing \(H\) by the effective horizon \((1-\gamma)^{-1}\) therefore provides only an order-wise comparison between the two bounds and does not yield our guarantees from their analysis.

\paragraph{Comparison with the lower bounds and remaining gap.}
The lower bounds of \citet{mortensen2026recursive} for both value and policy learning scale, up to logarithmic factors and universal constants, as
\begin{align*}
\widetilde{\Omega}\!\left(
\frac{SA\gamma^2}{\varepsilon^2\beta^2}
e^{|\beta|/(1-\gamma)}
\right).
\end{align*}

Consequently, our value-learning upper bound matches the lower bound of \citet{mortensen2026recursive} in its exponential dependence on $|\beta|/(1-\gamma)$, as well as in $S$, $A$, $\varepsilon$, and $\beta$, up to logarithmic factors. The remaining gap for value learning is polynomial in the effective horizon and is of order
$(1-\gamma)^{-2}$. 
In contrast, the upper bound of \citet{mortensen2026recursive} has an upper-to-lower gap containing an additional factor of order $e^{|\beta|/(1-\gamma)}$. Thus, our analysis closes the exponential part of this gap.

For policy learning, our upper bound likewise matches the lower bound of \citet{mortensen2026recursive} in the exponential factor. The remaining gap, considering the worst-case $|\Pi|=A^S$, is 
\begin{align*}
\widetilde{\cal O}\!\left(
\frac{1}{(1-\gamma)^2}
\min\!\left\{S,\frac{\gamma^2}{(1-\gamma)^2}\right\}
\right).
\end{align*}
Hence, for both value and policy learning, our analysis removes the exponential gap between the upper bounds of \citet{mortensen2026recursive} and their lower bounds.


\begin{remark}[Risk-neutral consistency]
Taking the limit in the PAC bounds of Theorems~\ref{thm:MBRSQVI_UB_Q}--\ref{thm:MBRSQVI_UB_policy} as $\beta\to0$ yields the corresponding sample complexity bounds for the risk-neutral case. In the small-$\varepsilon$ regime, the resulting bound for value learning scales as
\[
\widetilde{\cal O}\!\left(
\frac{SA\gamma^2}{\varepsilon^2(1-\gamma)^4}
\right),
\]
while the policy-learning bound scales as
\[
\widetilde{\cal O}\!\left(
\frac{SA\gamma^2}{\varepsilon^2(1-\gamma)^4}
\min\left\{S,\frac{\gamma^2}{(1-\gamma)^2}\right\}
\right).
\]
Thus, compared to the minimax-optimal risk-neutral rate, the resulting value-learning bound is off by a factor of $\gamma^2(1-\gamma)^{-1}$, while the policy-learning bound is off by a factor of $\gamma^2(1-\gamma)^{-1}\min\{S,\gamma^2(1-\gamma)^{-2}\}$.
\end{remark}

\section{Proofs: Sample Complexity Upper Bounds}
\label{sec:proof_UBs}
In this section, we prove Theorems \ref{thm:MBRSQVI_UB_Q} and \ref{thm:MBRSQVI_UB_policy}. We begin with basic and standard decompositions of the error terms associated to $V^{\pi_k}$ and $Q_k$. 

Let $\widehat{V}^*$ and $\widehat{Q}^*$ denote the optimal value and Q-value in $\widehat{M}$, respectively.  
Then, 
for any $(s,a)\in \cS\times\cA$, 
\begin{align}
\label{equation:decomposition_Qvalue}
    Q_k(s,a) &\geq Q^*(s,a) - \| \widehat{Q}^{\pi^*}-Q^* \| - \|Q_k-\widehat{Q}^*\|\, ,\\
\label{equation:decomposition_policy}
     V^{\pi_k}(s)  &\geq V^*(s) - \|V^{\pi_k} - \widehat{V}^{\pi_k} \| - \|\widehat{V}^{\pi^*}-V^* \|      -\|\widehat{V}^{\pi_k} - \widehat{V}^* \|\,,
\end{align}
where we recall that $\pi^*$ is an optimal policy in $M$. We also recall the following result from \citep[Lemma~2]{mortensen2026recursive}:
\begin{lemma}
\label{lemma:QVI_epsilon_k}
    Fix $\varepsilon>0$. Under MB-\RSQVI\ one has: (i) $\|Q_k-Q^*\|<\varepsilon$ if $k> \frac{-\log((1-\gamma)\varepsilon)}{\log(1/\gamma)}$; and (ii) $\|V^{\pi_k}-V^*\|<\varepsilon$ if $k> \frac{\log 2 - \log((1-\gamma)^2\varepsilon)}{\log(1/\gamma)}$.
\end{lemma}

\subsection{Proof of Theorem \ref{thm:MBRSQVI_UB_Q}}

\begin{proof}
Let $\varepsilon>0$ and $\beta\neq 0$. In view of the error decomposition in (\ref{equation:decomposition_Qvalue}), to establish $\varepsilon$-value-correctness it suffices to ensure $\| \widehat{Q}^{\pi^*}-Q^* \|\le \varepsilon/2$ and $\|Q_k-\widehat{Q}^*\|\le \varepsilon/2$. 
By Lemma \ref{lemma:QVI_epsilon_k}, we can have $\|Q_k-\widehat Q^*\|<\varepsilon/2$ using enough iterations of the optimization oracle. To ensure $\| \widehat{Q}^{\pi^*}-Q^* \|\le \varepsilon/2$, we use the following argument.

For a stationary deterministic policy \(\pi\), define 
\[
\widehat Q^\pi(s,a)
=
R(s,a)
-
\frac{\gamma}{\beta}
\log
\left(
\sum_{s'} \widehat P_{s,a}(s')e^{-\beta \widehat V^\pi(s')}
\right),
\]
and let $\widehat {\cT}^\pi$ denotes its corresponding Bellman operator, so that $\widehat Q^\pi = \widehat {\cT}^\pi \widehat Q^\pi$. 
Since \(\widehat {\cT}^\pi\) is a \(\gamma\)-contraction w.r.t.~$\|\cdot\|$ (see \citep[Lemma~?]{mortensen2026recursive}), we have
\[
\begin{aligned}
\|Q^\pi-\widehat Q^\pi\|
&= \|\cT^\pi Q^\pi - \widehat {\cT}^\pi \widehat Q^\pi\| \\
&\le \|\cT^\pi Q^\pi - \widehat {\cT}^\pi Q^\pi\| + \|\ \widehat {\cT}^\pi Q^\pi - \widehat {\cT}^\pi \widehat Q^\pi\| \\
&\le \|\cT^\pi Q^\pi - \widehat {\cT}^\pi Q^\pi\| +
\gamma \|Q^\pi-\widehat Q^\pi\|.
\end{aligned}
\]
Hence,  
\begin{align}
\|Q^\pi-\widehat Q^\pi\| &\le \frac{1}{1-\gamma} \|\cT^\pi Q^\pi - \widehat {\cT}^\pi Q^\pi\|\nonumber\\
&=   \frac{\gamma}{(1-\gamma)|\beta|}
    \max_{s,a} 
    \left|  \log   \frac{\sum_{s'}P_{s,a}(s')e^{-\beta V^\pi(s')}}{\sum_{s'}\widehat P_{s,a}(s')e^{-\beta V^\pi(s')}}    \right|.
    \label{eq:BelEr}
\end{align}

We consider two cases depending on the sign of $\beta$. 

\paragraph{Case 1: $\beta>0$.} 
For each \((s,a)\), define
\[
h_{s,a}
:=
\sum_{s'}P_{s,a}(s')e^{-\beta V^\pi(s')},
\qquad
\widehat h_{s,a}
:=
\sum_{s'}\widehat P_{s,a}(s')e^{-\beta V^\pi(s')}.
\]
Since \(0\le V^\pi(s')\le \frac{1}{1-\gamma}\) for all $s'\in \cS$,  $e^{-|\beta|/(1-\gamma)} \le e^{-\beta V^\pi(s')}\le 1$, so that   
\[
h_{s,a} \ge e^{-|\beta|/(1-\gamma)}.
\]

Let \(s'_1,\ldots,s'_N\) be the \(N\) samples collected on $(s,a)$ from the simulator. These are drawn i.i.d.~from
\(P(\cdot \mid s,a)\). Then, 
\[
\widehat h_{s,a}
= \frac{1}{N}\sum_{i=1}^N e^{-\beta V^\pi(s'_i)}. 
\]
The summands are i.i.d., bounded in \([e^{-|\beta|/(1-\gamma)},1]\), and have mean
\(h_{s,a}\). Applying Bernstein's inequality \ref{lem:Bernstein}, we have that with probability at least
\(1-\frac{\delta}{SA}\),
\begin{align*}
\left|
\widehat h_{s,a} - h_{s,a}
\right|
&\le
\sqrt{
\frac{
2\operatorname{Var}_{P_{s,a}}
\left(e^{-\beta V^\pi(s')}\right)
\log(2SA/\delta)
}{N}
}
+ \frac{2\left(1-e^{-|\beta|/(1-\gamma)}\right)\log(2SA/\delta)}{3N}
\end{align*}
To control $\operatorname{Var}_{P_{s,a}}$, we use the Bhatia--Davis inequality (Lemma \ref{lem:Bhatia-Davis}): 
\[
\operatorname{Var}_{P_{s,a}}\left(e^{-\beta V^\pi(s')}\right)
\leq
(1-h_{s,a})\left(h_{s,a}-e^{-|\beta|/(1-\gamma)}\right).
\]
We note that a closely related variance bound for exponential transforms is given in \citet[Lemma~24]{essakine26tight}, where the Bhatia--Davis inequality is used. We apply the argument here to the discounted setting. Hence, with probability at least $1-\frac{\delta}{SA}$,
\begin{align*}
\left|\widehat h_{s,a}-h_{s,a}\right|
&\leq
\sqrt{\frac{2(1-h_{s,a})
\left(h_{s,a}-e^{-|\beta|/(1-\gamma)}\right)
\log(2SA/\delta)}{N}}
+
\frac{2\left(1-e^{-|\beta|/(1-\gamma)}\right)
\log(2SA/\delta)}{3N}.
\end{align*}

We have
\[
\max_{x\in[e^{-|\beta|/(1-\gamma)},1]}
\frac{(1-x)\left(x-e^{-|\beta|/(1-\gamma)}\right)}{x^2}
=
\frac{1}{4}\left(
e^{|\beta|/(1-\gamma)}+e^{-|\beta|/(1-\gamma)}-2
\right).
\]
Indeed, letting $m=e^{-|\beta|/(1-\gamma)}$ and
\[
g(x)=\frac{(1-x)(x-m)}{x^2},
\]
we have
\[
g'(x)=\frac{2m-(1+m)x}{x^3}.
\]
Thus, $g$ is maximized at $x=2m/(1+m)$, and the substitution gives
\[
g\!\left(\frac{2m}{1+m}\right)
=
\frac{(1-m)^2}{4m}
=
\frac{1}{4}\left(
e^{|\beta|/(1-\gamma)}+e^{-|\beta|/(1-\gamma)}-2
\right).
\]
This is the same bound appearing in \citet[Lemma~24]{essakine26tight}.
Moreover,
\[
\frac{1-e^{-|\beta|/(1-\gamma)}}{h_{s,a}}
\leq e^{|\beta|/(1-\gamma)}-1.
\]
Hence, if
\[
N\geq \log\!\left(\frac{2SA}{\delta}\right)
\max\left\{
8\left(e^{|\beta|/(1-\gamma)}+e^{-|\beta|/(1-\gamma)}-2\right),\,
\frac{8}{3}\left(e^{|\beta|/(1-\gamma)}-1\right)
\right\},
\]
then $|\widehat h_{s,a}-h_{s,a}|\leq h_{s,a}/2$. 
Since
\[
e^{|\beta|/(1-\gamma)}+e^{-|\beta|/(1-\gamma)}-2
\leq e^{|\beta|/(1-\gamma)}-1,
\]
the condition
\[
N\geq
8\left(e^{|\beta|/(1-\gamma)}-1\right)
\log\!\left(\frac{2SA}{\delta}\right)
\]
ensures that $|\widehat h_{s,a}-h_{s,a}|\leq h_{s,a}/2$.

Now, applying Lemma \ref{lem:log-inequality} gives
\begin{align*}
\big|\log \widehat h_{s,a}-\log h_{s,a}\big|
&\leq
\sqrt{\frac{8(1-h_{s,a})\left(h_{s,a}-e^{-|\beta|/(1-\gamma)}\right)\log(2SA/\delta)}{Nh_{s,a}^2}}
+
\frac{4\left(1-e^{-|\beta|/(1-\gamma)}\right)\log(2SA/\delta)}{3Nh_{s,a}}\\
&\leq
\sqrt{\frac{2\left(e^{|\beta|/(1-\gamma)}+e^{-|\beta|/(1-\gamma)}-2\right)\log(2SA/\delta)}{N}}
+
\frac{4\left(e^{|\beta|/(1-\gamma)}-1\right)\log(2SA/\delta)}{3N}.
\end{align*} Putting together with equation~\eqref{eq:BelEr} and taking a union bound over all \((s,a)\in\mathcal S\times\mathcal A\),
we have that with probability at least \(1-\delta\),
\begin{align*}
\|Q^\pi-\widehat Q^\pi\|
&\leq
\frac{\gamma}{(1-\gamma)|\beta|}
\Bigg[
\sqrt{\frac{2\left(
 e^{|\beta|/(1-\gamma)}+e^{-|\beta|/(1-\gamma)}-2
\right)\log(2SA/\delta)}{N}}
+
\frac{4\left(e^{|\beta|/(1-\gamma)}-1\right)
\log(2SA/\delta)}{3N}
\Bigg].
\end{align*}

Applying the bound above for $\pi = \pi^\star$, we have that the condition $\|Q^*-\widehat Q^*\|\leq \varepsilon/2$ can be ensured by taking
\begin{align*}
N\geq \log\!\left(\frac{2SA}{\delta}\right)
\max\Bigg\{&
8\left(e^{|\beta|/(1-\gamma)}-1\right),
\frac{32\gamma^2
\left(e^{|\beta|/(1-\gamma)}+e^{-|\beta|/(1-\gamma)}-2\right)}
{\varepsilon^2(1-\gamma)^2\beta^2},
\frac{16\gamma\left(e^{|\beta|/(1-\gamma)}-1\right)}
{3\varepsilon(1-\gamma)|\beta|}
\Bigg\},
\end{align*}
Combining this with $\|Q_k-\widehat Q^*\|\leq \varepsilon/2$ and noting that $T=SAN$ proves Theorem \ref{thm:MBRSQVI_UB_Q}.

\paragraph{Case 2: $\beta<0$.}
For each \((s,a)\), define
\[
h_{s,a}
:=
\sum_{s'}P_{s,a}(s')e^{-|\beta|\left(\frac{1}{1-\gamma}-V^\pi(s')\right)},
\qquad
\widehat h_{s,a}
:=
\sum_{s'}\widehat P_{s,a}(s')e^{-|\beta|\left(\frac{1}{1-\gamma}-V^\pi(s')\right)}.
\]
Compared to the case of $\beta>0$, the exponents have some shift. However, this does not change the logarithmic ratio because it multiplies both \(h_{s,a}\) and
\(\widehat h_{s,a}\) by the same constant \(e^{-|\beta|/(1-\gamma)}\). 

Note that 
\[
e^{-|\beta|/(1-\gamma)}
\le
e^{-|\beta|\left(\frac{1}{1-\gamma}-V^\pi(s')\right)}
\le
1.
\]
Hence, $e^{-|\beta|/(1-\gamma)} \le h_{s,a} \le 1$. 
Since the shifted exponential term lies in the same interval as in Case 1, the same application of Lemma \ref{lem:Bhatia-Davis} applies:
\[
\operatorname{Var}_{P_{s,a}}\!\left(
 e^{-|\beta|\left(\frac{1}{1-\gamma}-V^\pi(s')\right)}
\right)
\leq
(1-h_{s,a})\left(h_{s,a}-e^{-|\beta|/(1-\gamma)}\right).
\]
Hence, the variance term admits exactly the same control as in Case 1. Together with the fact that the shift does not change the logarithmic ratio, the remainder of the argument is identical, and the same sample-complexity bound follows for $\beta<0$.
\end{proof}

\subsubsection{Proof of Theorem \ref{thm:MBRSQVI_UB_policy}}

\begin{proof}
Let $\varepsilon>0$ and $\beta\neq 0$.
In view of the error decomposition in (\ref{equation:decomposition_policy}), to establish $\varepsilon$-policy-correctness it suffices to require: (i) $\|\widehat{V}^{\pi_k}-\widehat{V}^*\|\leq \varepsilon/3$, (ii) $\|\widehat{V}^{\pi^*}-V^*\|\leq \varepsilon/3$, and (iii) $\|V^{\pi_k}-\widehat{V}^{\pi_k}\|\leq \varepsilon/3$.
For (i), we can have $\|\widehat{V}^{\pi_k}-\widehat{V}^*\|\leq \varepsilon/3$ using enough iterations of the optimization oracle as a consequence of Lemma \ref{lemma:QVI_epsilon_k}. To control (ii) and (iii), let us define the event
$$
E:=\cap_{\pi\in \Pi}\big\{\|V^{\pi} - \widehat{V}^\pi\| \le \varepsilon/3\}.
$$
We can show that $\P(E)\ge 1-\delta$ for sufficiently large $N$. Indeed, for a given $\pi\in \Pi$, the fixed-policy bound established in the proof of Theorem \ref{thm:MBRSQVI_UB_Q}, with $\varepsilon$ replaced by $\varepsilon/3$ and $\delta$ replaced by $\delta/|\Pi|$, gives
$
\P\!\left(\|V^\pi-\widehat V^\pi\|>\varepsilon/3\right)\leq \delta/|\Pi|
$
whenever
\begin{align*}
N\geq \log\!\bigg(\frac{2SA|\Pi|}{\delta}\bigg)
\max\Bigg\{&
8\left(e^{|\beta|/(1-\gamma)}-1\right),
\frac{72\gamma^2
\left(e^{|\beta|/(1-\gamma)}+e^{-|\beta|/(1-\gamma)}-2\right)}
{\varepsilon^2(1-\gamma)^2\beta^2},
\frac{8\gamma\left(e^{|\beta|/(1-\gamma)}-1\right)}
{\varepsilon(1-\gamma)|\beta|}
\Bigg\}.
\end{align*}
Hence,
\begin{align*}
\P(E^c)
&=\P\!\left(\exists\pi\in\Pi:\|V^\pi-\widehat V^\pi\|>\varepsilon/3\right)
\leq \sum_{\pi\in\Pi}\P\!\left(\|V^\pi-\widehat V^\pi\|>\varepsilon/3\right)
\leq \delta,
\end{align*}
so that $\P(E)\geq 1-\delta$.
It is evident that conditioned on $E$, (ii) and (iii) hold. Hence, with probability greater than $1-\delta$, $\varepsilon$-policy-correctness is maintained using a total number of model calls of
\begin{align}
\label{eq:T_1}
T_{1,\delta}
=SA\log\!\bigg(\frac{2SA|\Pi|}{\delta}\bigg)
\max\Bigg\{&
8\left(e^{|\beta|/(1-\gamma)}-1\right),\nonumber
\frac{72\gamma^2
\left(e^{|\beta|/(1-\gamma)}+e^{-|\beta|/(1-\gamma)}-2\right)}
{\varepsilon^2(1-\gamma)^2\beta^2},\nonumber\\
&\frac{8\gamma\left(e^{|\beta|/(1-\gamma)}-1\right)}
{\varepsilon(1-\gamma)|\beta|}
\Bigg\}.
\end{align}

To establish the second bound, first observe that Theorem \ref{thm:MBRSQVI_UB_Q} implies that $\|\widehat{V}^{\pi_k} - V^*\|\leq \varepsilon$ with probability exceeding $1-\delta$ if
\begin{align*}
T\geq SA\log\!\bigg(\frac{2SA}{\delta}\bigg)
\max\Bigg\{&
8\left(e^{|\beta|/(1-\gamma)}-1\right),
\frac{32\gamma^2
\left(e^{|\beta|/(1-\gamma)}+e^{-|\beta|/(1-\gamma)}-2\right)}
{\varepsilon^2(1-\gamma)^2\beta^2},
\frac{16\gamma\left(e^{|\beta|/(1-\gamma)}-1\right)}
{3\varepsilon(1-\gamma)|\beta|}
\Bigg\};
\end{align*}
this is verified by noting that
$$
\|\widehat{V}^{\pi_k} - V^*\|
= \max_s \big|\max_a \widehat Q^{\pi_k}(s,a) - \max_a Q^{*}(s,a)\big|
\leq \max_{s,a} \big|\widehat Q^{\pi_k}(s,a) - Q^{*}(s,a)\big|
=\|Q_k - Q^*\| .
$$

To proceed, we make use of the following lemma:

\begin{lemma}[{\cite[Lemma~5]{mortensen2026recursive}}]
\label{lem:DeteriorationOfGreedyPolicy}
    Let $\alpha>0$. Let $\overline{V}\in \R^S$ be a value function obeying $\|V^*-\overline{V}\|<\alpha$, and $\pi^G:=\argmax_a[R(s,a)+\gamma \rho_{s,a}(\overline{V}(s'))]$ be the greedy policy with respect to $\overline{V}$. Then,
    $\|V^*-V^{\pi_G}\|\leq \frac{2\gamma}{1-\gamma}\alpha$\,.
\end{lemma}

By construction, $\pi_k$ is the greedy policy with respect to $\widehat{V}^{\pi_k}$. Now applying Lemma \ref{lem:DeteriorationOfGreedyPolicy} with $\overline V = \widehat{V}^{\pi_k}$, we have that
$$
\|V^*-V^{\pi_k}\|\leq \frac{2\gamma}{1-\gamma}\varepsilon,
$$
which yields the following bound that holds with probability larger than $1-\delta$:
\begin{align}
\label{eq:T_2}
T_{2,\delta}
=SA\log\!\bigg(\frac{2SA}{\delta}\bigg)
\max\Bigg\{&
8\left(e^{|\beta|/(1-\gamma)}-1\right),\nonumber
\frac{128\gamma^4
\left(e^{|\beta|/(1-\gamma)}+e^{-|\beta|/(1-\gamma)}-2\right)}
{\varepsilon^2(1-\gamma)^4\beta^2},\nonumber\\
&\frac{32\gamma^2\left(e^{|\beta|/(1-\gamma)}-1\right)}
{3\varepsilon(1-\gamma)^2|\beta|}
\Bigg\}.
\end{align}

\paragraph{Final Bound.}To derive the final bound, we put together (\ref{eq:T_1}) and (\ref{eq:T_2}), while suitably adjusting the error probabilities. Therefore,  
$\varepsilon$-policy-correctness is guaranteed with probability exceeding $1-\delta$ if $T\ge \min\{T_{1,\delta/2},T_{2,\delta/2}\}$.
\end{proof}

\section{Conclusions}\label{Section:Conclusions}

We studied sample complexities of value learning and policy learning in finite discounted MDPs under recursive ERM, assuming that a generative model of the environment is available. 
We derived refined PAC-type sample complexity bounds for MB-\RSQVI\ \citep{mortensen2026recursive} that exhibit the same exponential dependence as in the corresponding minimax lower bounds for this setting. 
An interesting direction for future work is to close between the gap in the lower and upper bounds for this setting. Another promising direction is the development of model-free algorithms for this setting, and to study MDPs with function approximation (e.g., \cite{zhou2021provably}). 

\subsubsection*{Acknowledgments}
The authors would like to acknowledge the support from Independent Research Fund Denmark, grant number 1026-00397B.


\bibliography{references}
\bibliographystyle{plainnat}

\section{Supporting Lemmas}

\begin{lemma}
\label{lem:log-inequality}
Consider $x,y>0$. If
$|y-x|\le \frac{1}{2} x$, then
\[
|\log y-\log x|
\le
\frac{2|y-x|}{x}.
\]
\end{lemma}

\begin{proof}
First, we consider $x>y$.  By the mean value theorem applied to \(x\mapsto \log x\), there exists some \(\xi\in[y,x]\) such that
\[
\log y-\log x=
\frac{1}{\xi}(x- y).
\]
By the assumption, $y\ge\frac12 x$, so that any \(\xi\in[y,x]\) satisfies
$\xi\ge \frac{1}{2}x$, 
thus giving 
\[
|\log y-\log x|
=
\frac{|y-x|}{\xi}
\le
\frac{2|y-x|}{x}.
\]
If $x<y$, then $\xi\ge$ will hold trivially, and the lemma follows. 
\end{proof}

\begin{lemma}[{Bernstein's Inequality}] \label{lem:Bernstein}
Let $(X_1,\ldots,X_n)$ be independent, zero-mean random variables such that
$|X_i|\le b$ almost surely, and define $V:=\sum_{i=1}^n \mathbb{E}[X_i^2]$.
Then, for any $\delta\in(0,1)$, with probability at least \(1-\delta\),
$$
\sum_{i=1}^n X_i
\le
\sqrt{2V\log(1/\delta)} + \frac{b}{3}\log(1/\delta).
$$
\end{lemma}

\begin{lemma}[{The Bhatia--Davis Inequality \citep{bhatia2000better}}] \label{lem:Bhatia-Davis}
Let $X$ be a random variable such that $m \le X \le M$ almost surely, and let $\mu = \mathbb{E}[X]$. Then
\[
\operatorname{Var}(X) \le (M-\mu)(\mu-m).
\]
In particular,
\[
\operatorname{Var}(X) \le \frac{(M-m)^2}{4}.
\]
\end{lemma}

\end{document}